\documentclass[]{elsarticle}

\usepackage{amsmath,mathtools}
\usepackage{amsfonts}
\usepackage{booktabs}
\usepackage{graphicx}
\usepackage[caption=false,font=footnotesize]{subfig}
\usepackage{paralist}
\usepackage{psfrag} 
\usepackage{xcolor}
\usepackage{url}

\usepackage[caption=false,font=footnotesize]{subfig} 
\graphicspath{{./Figures/}} 

\DeclareMathOperator{\ev}{E}

\renewcommand{\imath}{j}
\newcommand{\noisy}[1]{#1^{\epsilon}}
\newcommand{\numberthis}{\addtocounter{equation}{1}\tag{\theequation}}

\newcommand{\He}{\mathit{He}} 

\newtheorem{theorem}{Theorem}
\newtheorem{lemma}{Lemma}
\newproof{proof}{Proof}

\journal{}

\begin{document}

\begin{frontmatter}

\title{Rotation Invariant Angular Descriptor Via A Bandlimited Gaussian-like Kernel}

\author{Michael T. McCann, Matthew Fickus, Jelena Kova\v{c}evi\'c}





\begin{abstract}
We present a new smooth, Gaussian-like kernel that allows the kernel density estimate for an angular distribution to be exactly represented by a finite number of its Fourier series coefficients.
Distributions of angular quantities, such as gradients, are a central part of several state-of-the-art image processing algorithms, but these distributions are usually described via histograms and therefore lack rotation invariance due to binning artifacts.
Replacing histograming with kernel density estimation removes these binning artifacts and can provide a finite-dimensional descriptor of the distribution, provided that the kernel is selected to be bandlimited.
In this paper, we present a new band-limited kernel that has the added advantage of being Gaussian-like in the angular domain.
We then show that it compares favorably to gradient histograms for patch matching, person detection, and texture segmentation.
\end{abstract}


\end{frontmatter}

\setlength\arraycolsep{2pt} 

\section{Introduction}
Histograms of angular quantities are a key component of many of the most successful algorithms for a variety of image processing tasks.
For example, SIFT \cite{Lowe:99}, along with some of its variants including GLOH~\cite{MikolajczykS:05}, SIFT+GC~\cite{MortensenDS:05}, and CSIFT~\cite{Abdel-HakimF:06} (but not SURF~\cite{BayETV:08} or PCA-SIFT~\cite{KeS:04}), use histograms of local gradient angles to form a keypoint descriptor.
SIFT descriptors are widely used, with applications including medical image registration~\cite{SotirasDP:13}, human activity analysis~\cite{AggarwalR:11}, and object recognition~\cite{RamananN:12}.
HOG~\cite{DalalT:05} and its extensions, such as part-based models~\cite{FelzenszwalbGMR:10}, calculate local gradient histograms at every point in an image and are useful in human~\cite{GeronimoLSG:10} and object~\cite{EveringhamGWWZ:10} detection.

Despite their widespread use in vision, histograms have a fundamental weakness when estimating distributions of angular quantities because they rely on binning and are therefore not invariant to rotation.
That is, a rotation of the input angles results in a rotation of the histogram plus distortion; see Figure~\ref{fig:binning}.
This problem affects even methods that attempt to be invariant to rotation.
For example, SIFT~\cite{Lowe:99} builds histograms with respect to a dominant angle, but this dominant angle is itself estimated from an angular histogram.
The radial gradients computed in RIFF~\cite{TakacsCTCGG:13} are invariant to rotation, but are collected into angular histograms which are not.

\begin{figure}%
\centering
\includegraphics[page=1]{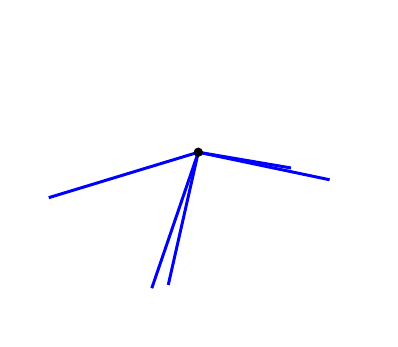}%
\includegraphics[page=2]{paper-pics_static.pdf}%
\includegraphics[page=3]{paper-pics_static.pdf}\\%
\includegraphics[page=4]{paper-pics_static.pdf}%
\includegraphics[page=5]{paper-pics_static.pdf}%
\includegraphics[page=6]{paper-pics_static.pdf}%
\caption{Top row: A weighted set of angles, its 20-bin histogram, and the estimate based on the proposed method (FS-KDE) using 20 numbers.
Bottom row: The same set of angles rotated counter-clockwise by $45^{\circ}$, its histogram, and FS-KDE.
While the rotation distorts the histogram, it only causes a corresponding rotation in the FS-KDE.}%
\label{fig:binning}%
\end{figure}

A rotation-invariant alternative to the histogram is the kernel density estimate (KDE)~\cite{Wasserman:04}, which estimates a continuous distribution from its samples by putting a lump (kernel) of density at the location of each sample.
KDEs create smooth estimates and, under certain assumptions, converge to the correct distribution with fewer samples than histograms \cite{Wasserman:04};
however, KDEs are not useful as descriptors in image processing, because evaluating the KDE at a point requires all of the samples to be stored in memory and there is no straightforward way to compute distances between KDEs.
A variant of the KDE that helps to address these limitations is characteristic function estimation~\cite{feuerverger_empirical_1977}, wherein the characteristic function (or, in the language of signal processing, Fourier series) of a distribution is estimated, rather than its angular-domain version.
For angular distributions, the estimated Fourier series is discrete and can be truncated to form a finite-length descriptor.
Reference \cite{LiuSSBPBR:2013} explores the use this type of descriptor as a replacement for histograms inside of HOG~\cite{DalalT:05}.
The problem with this truncation is that it is equivalent to convolving the angular kernel of a KDE with a sinc function.
These kernels may then have undesirable angular domain properties, such as attaining negative values or being non-monotonic on the intervals $[-\pi, 0]$ and $[0, \pi]$.
For example, since \cite{LiuSSBPBR:2013} uses a Dirac kernel in the angular domain and then truncates the Fourier series, the effective angular kernel is a sinc function. 

In this work, we present a new kernel designed to have good properties in the angular domain while simultaneously being band-limited in the frequency domain, meaning that its Fourier series has a finite number of non-zero terms and can therefore be used directly as a descriptor without truncation.



\section{Fourier Series Kernel Density Estimation}
\label{sec:theory}
We call the method of representing an angular KDE via its Fourier series the \emph{Fourier Series-Kernel Density Estimate} (FS-KDE).
In this section, we develop our notation for the FS-KDE and describe its properties.

\subsection{Definition of the FS-KDE}
\label{sec:derivation}

Given an angle-weight pair, $(\Theta, W)$, consisting of a set of angles, $\Theta = \begin{Bmatrix}\theta_0,&\theta_1, & \dots, & \theta_{N-1}\end{Bmatrix}$, and a set of positive scalar weights, $W = \begin{Bmatrix} w_0, & w_1, & \dots, & w_{N-1}\end{Bmatrix}$,  we form a KDE, $f: \begin{bmatrix}-\pi, & \pi \end{bmatrix} \rightarrow \mathbb{R}$, of their underlying distribution as a sum of kernels,
\begin{equation*}
f(\theta) = \frac{1}{N}\sum_{n = 0}^{N-1} w_n h\left(\theta - \theta_n \right),
\end{equation*}
where the kernel, $h(\theta)$, is a positive function that integrates to one.%
\footnote{For greater flexibility, we do not require $f$ to integrate to one; we thus use the term \emph{distribution} loosely.}
For example, the angle-weight pair might come from the angles and magnitudes of the gradients in an image.


We can then expand $h$ in terms of its Fourier series and rearrange terms,
\begin{align*}
	f(\theta) &= \frac{1}{N}\sum_{n = 0}^{N-1} w_n \sum_{k = -\infty}^{\infty} H_k  e^{ \imath k (\theta - \theta_n)} \\
	&\overset{(a)}{=} \frac{1}{N}\sum_{n = 0}^{N-1} w_n \sum_{k = -K}^{K} H_k  e^{ \imath k (\theta - \theta_n)} \\
		&=\sum_{k = -K}^{K}  \underbrace{\left( \frac{H_k }{N} \sum_{n = 0}^{N-1} w_n   e^{- \imath k \theta_n} \right)}_{F_k}   e^{ \imath k \theta} \numberthis \label{eq:FS-KDE_FK},
\end{align*}
where (a) holds for bandlimited kernels.
Equation \eqref{eq:FS-KDE_FK} is the expression of $f$ in terms of its Fourier series coefficients, $F_k$.
We denote the relationship between $f(\theta)$ and $F_k$ as  $f(\theta) \overset{FS}{\leftrightarrow} F_k$.
From \eqref{eq:FS-KDE_FK}, we see that $f$ is bandlimited: it has $2K+1$ non-zero Fourier series coefficients.
We also see that $F_{-k}$ is the complex conjugate of $F_{k}$, so, in practice, only $K+1$ complex values must be computed and stored to represent $f$.
Thus, an FS-KDE of order $K$ takes the same amount of storage as a histogram with $2(K+1)$ bins.

\subsection{Properties}
\label{sec:props}
We now discuss some useful properties of the FS-KDE.
First, since $F$ is simply the Fourier series representation of $f$, we can leverage all of the properties of the Fourier series~\cite{VetterliKG:12}.
Of specific interest here are linearity, $\alpha f(\theta) + \beta g(\theta) \overset{FS}{\longleftrightarrow}  \alpha F_k + \beta G_k$, and  Parseval's equality, $||f||^2 = 2\pi ||F||^2$.
Together, these mean that the distance between two FS-KDEs, $||f-g||^2$, can be computed as the finite sum $2 \pi || F - G ||^2$.

The FS-KDE is rotation invariant in the sense that a rotation of the angles in the angle-weight pair results in a corresponding rotation in the FS-KDE.
To be more precise, begin with an angle-weight pair $(\Theta, W)$.
Form its rotation, $(\Theta_\phi, W)$, where $\Theta_\phi = \begin{Bmatrix}\theta_0 + \phi, & \theta_1+ \phi, & \dots, & \theta_{N-1} + \phi\end{Bmatrix}$.
If $F$ is the FS-KDE for $(\Theta, W$) and $F_\phi$ is the FS-KDE for $(\Theta_\phi, W)$, then from \eqref{eq:FS-KDE_FK},
\begin{align*}
F_{\phi, k} &= \frac{H_k}{N}  \sum_{n = 0}^{N-1} w_n   e^{- \imath k (\theta_n + \phi)} \\
&= e^{-\imath k \phi} \frac{H_k}{N} \sum_{n = 0}^{N-1} w_n e^{-\imath k \theta_n}   \\
&= e^{-\imath k \phi} F_{K}. \numberthis \label{eq:rotatedF_k}
\end{align*}
By the shift in time property of Fourier series, \eqref{eq:rotatedF_k} means that $f_\phi$ is equal to $f$ circularly shifted by $\phi$.
Thus, a rotation in the input angles has caused a corresponding rotation in the KDE.


		
\subsection{FS-KDEs for Images}
Several computer vision algorithms estimate local angular distributions (usually via histograms) for every location in an image.
For example, this is the approach of deformable parts models~\cite{FelzenszwalbGMR:10} for object detection.
In this section, we describe efficient computation of local FS-KDE estimates on images via linear filtering.

Let $(\Theta(x), W(x))$ be a weighted angular image, where $\Theta : X \rightarrow [-\pi, \pi]$ is an image of angles and $ W: X \rightarrow \mathbb{R}$ is a corresponding image of weights, where $X$ is a discrete set of pixel locations (e.g. $\mathbb{Z}_{1200} \times \mathbb{Z}_{1600}$).
For example, $(\Theta(x), W(x))$ may be formed from computing the gradient of an intensity image.
(Note that we write the argument $x$ in $(\Theta(x), W(x))$ to make a distinction between the weighted angular image and the angle-weight pair we introduced in Section~\ref{sec:derivation}.)

We aim to compute a KDE around each point $x \in X$, with a neighborhood defined by $\varphi$, a positive window function with $||\varphi||_1 = 1$.
We define the FS-KDE for $(\Theta(x), W(x))$ at location a $x$ and angle $\theta$ as
\begin{equation*}
f(x, \theta) = \sum_{y\in X} W(y) h\left(\theta - \Theta(y)\right)  \varphi(x - y),
\end{equation*}
Then, following a similar procedure as in Section~\ref{sec:derivation}, we arrive at an expression for $f(x, \theta)$ in terms of its Fourier series coefficients,
\begin{equation*}
f(x, \theta) = \sum_{k=-K}^{K} F_k(x) e^{\imath k \theta},
\end{equation*}
where
$F_k(x) = H_k \left( W(\cdot) e^{-\imath k  \Theta(\cdot) }  \ast \varphi \right) (x)$,
the symbol $\ast$ denotes discrete time convolution, and $W(\cdot) e^{-\imath k \Theta(\cdot)}$ is computed pointwise.
This means that local FS-KDEs of order $K$ can be computed via $K+1$ complex filtering operations. 


\section{Bandlimited Gaussian-like Kernel}
In this section we present a new Guassian-like kernel for use in FS-KDEs, which we call the $\cos^{2K}$ kernel.

\subsection{Kernel Selection}
\label{sec:kernelSel}
Any kernel that has a bandlimited Fourier series can be used to form an FS-KDE using \eqref{eq:FS-KDE_FK},
however we argue for the following additional requirements to make the kernel reasonable for estimation of angular distributions:
\begin{inparaenum}[(1)]
\item the kernel should be real; \label{req:real}
\item the kernel should be non-negative; \label{req:pos}
\item the kernel should be an even function;\label{req:even}
\item the kernel should integrate to one; and\label{req:one}
\item the kernel should take the value zero at $\pi$. \label{req:zero}
\end{inparaenum}

We propose the $\cos^{2K}$ kernel, 
\begin{equation}
h(\theta) = C_K \cos^{2K} \left(\frac{\theta}{2} \right), \label{eq:BGLK}
\end{equation}
where $C_K$ is a normalizing constant and $K$, which we call the \emph{order}, controls the width of the kernel (Figure~\ref{fig:kernels}).
The $\cos^{2K}$ kernel clearly satisfies requirements \ref{req:real}-\ref{req:zero} above.
We will now show that the it is also bandlimited and Guassian-like.

\begin{figure}
\includegraphics[page=7]{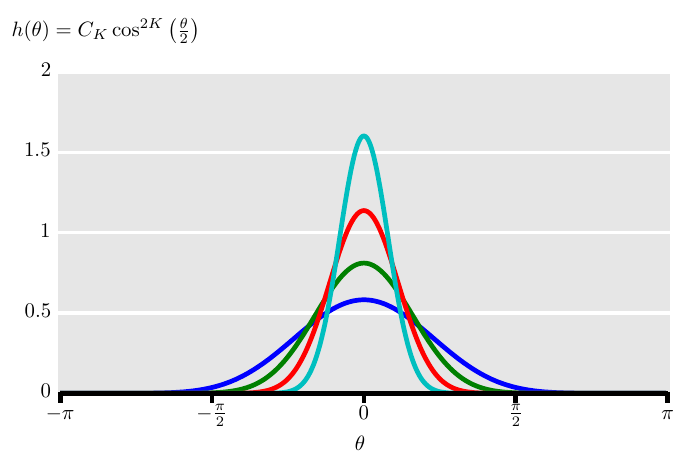}%
\caption{Examples of the proposed $\cos^{2K}$ kernel kernel for $K =$ 4, 8, 16, and 32, with $K=32$ being the narrowest of the kernels shown. 
Each kernel is $2\pi$-periodic and integrates to one.
As $K$ increases, the kernels become sharper.}
\label{fig:kernels}
\end{figure}

{\bf Bandlimited.} 
We can rearrange \eqref{eq:BGLK} to reveal its Fourier series coefficients,

\begin{align*}
h(\theta) &\overset{(a)}{=} C_K \left( \frac{e^{ \frac{\imath}{2}\theta} + e^{-\frac{\imath}{2}\theta}}{2} \right)^{2K} \\
&\overset{(b)}{=} \frac{C_K}{2^{2K}} \sum_{p=0}^{2K} \binom{2K}{p} e^{\frac{\imath}{2} p \theta} e^{-\frac{\imath}{2}(2K-p)\theta}  \\
&\overset{(c)}{=} \sum_{p=0}^{2K} \frac{C_K}{2^{2K}} \binom{2K}{p} e^{\imath (p - K) \theta} \\
&\overset{(d)}{=} \sum_{k=-K}^{K} \underbrace{\left( \frac{C_K}{2^{2K}}  \binom{2K}{K+k} \right) }_{H_k} e^{ \imath k \theta}  \numberthis \label{eq:BGLK_HK}
\end{align*} 
where $(a)$ follows from Euler's formula; $(b)$ from the binomial theorem; $(c)$ from an interchange of finite sums and combining the exponents; and $(d)$ from the substitution $k = p - K$.
This expression for $h(\theta)$ confirms that it is bandlimited.


{\bf Gaussian-like.}
The Gaussian kernel,
\begin{equation*}
  g(x) = \frac{1}{\sqrt{2\pi\sigma^2}} e^{ -\frac{x^2}{2\sigma^2} }.
\end{equation*}
is ubiquitous, but does not work as an FS-KDE kernel because it is neither bandlimited nor defined on a circular domain.%
\footnote{There are several adaptations of the Gaussian to a circular domain, including the von Mises and circularly extended Gaussian, neither of which are bandlimited and so are not suitable FS-KDE kernels.}
We will show that the $\cos^{2K}$ is Gaussian-like in that its derivatives behave in a similar way.
The derivatives of the Gaussian are
\begin{equation*}
  g^{(n)}(x) = \frac{d^n}{(dx)^n} g(x) = (-1)^n \He_n \left( \frac{x}{\sigma} \right) g(x),
\end{equation*}
where $\He_n(x)$ is the $n$th order Hermite polynomial in $x$. 
This is useful because we know that $\He_n(x)$ has $n$ real roots, each with multiplicity one.
Therefore, $g^{(n)}(x) = 0$ for $n$ values of $x$ (and also tends to zero as $|x|$ tends to infinity).

We want to show that the $\cos^{2K}$ kernel  \eqref{eq:BGLK} is Gaussian-like, i.e. $h^{(n)}(x) = 0$ for $n$ values of $\theta \ne \pi$ (our logic for ignore zeros at $\pi$ is by the analogy $\lim_{x \to \infty} g(x) \sim h(\pi)$).
By the definition of the Fourier series and the fact that $h(\theta)$ is bandlimited, we have
\begin{equation*}
  h(\theta) = \sum_{k = -K}^{K} H_k e^{\imath k \theta}.
\end{equation*}
We know that the $H_k$ are real and $H_k = H_{-k}$ because $h(\theta)$ is even and real.
Thus $h$ is just a polynomial of degree $2K$, $h(\theta) = e^{-\imath K \theta}P(e^{\imath k \theta})$.
By the fundamental theorem of algebra, the number of roots of $P(z)$ is $2K$; if we call those roots $z_0$, $z_1$, \dots, $z_{2K-1}$, then the zeros of $h$ correspond to the unit-modulus $z_i$s: $h(\arg(z_i)) = 0$ for each $z_i$ with $|z_i| = 1$,
where we use $\arg (z)$ to denote the argument of the complex number $z$. 

The derivatives of $h(\theta)$ are
\begin{equation*}
  h^{(n)}(\theta) = \frac{d^n}{(d\theta)^n} h(\theta) = \sum_{k = -K}^{K} (\imath k)^n H_k e^{\imath k \theta},
\end{equation*}
These derivatives are also polynomials of degree $2K$, which we call $P^{(n)}$, with the same relationship between the roots of the polynomial and the zeros of $h^{(n)}$ as for $h$ and $P$.

We see from \eqref{eq:BGLK} that all of the zeros of $h$ are at $\theta = \pi$, meaning that $P(z)$ has a root at $z=-1$ with multiplicity $2K$.
This also means that each $P^{(n)}$ has a root at $z=-1$ with multiplicity $2K-n$ (by the chain rule).
Thus the number of zeroes of $h^{(n)}(\theta)$ for $\theta \ne \pi$ is less than or equal to $n$.

On the other hand, because $h(-\pi) = h(\pi) = 0$, the mean value theorem guarantees that there exists a $\theta_0 \in (-\pi,  \pi)$ such that $h^{(1)}(\theta_0) = 0$.
Repeating the same argument gives $\theta_{1,0} \in (-\pi, \theta_0)$ and $\theta_{1,1} \in (\theta_0, \pi)$ such that  $h^{(2)}(\theta_{1,0}) = h^{(2)}(\theta_{1,1}) = 0$ and so on for each $h^{(n)}$.
Thus the number of zeroes of $h^{(n)}(\theta)$ for $\theta \ne \pi$ is greater than or equal to $n$.

Combining the inequalities from the previous two paragraphs, we have $h^{(n)}(x) = 0$ for $n$ values of $\theta \ne \pi$.

\subsection{Practical Considerations}
In this section, we explore a few practical considerations that must be taken into account when computing FS-KDE using our $\cos^{2K}$ kernel, including calculation of $C_K$, a normal approximation to \eqref{eq:FS-KDE_FK}, and creating approximate FS-KDEs via truncation.

\subsubsection{Normalization}
We have not yet calculated $C_K$, the normalizing constant for the kernel $h$ in \eqref{eq:BGLK}.
We do this via \eqref{eq:BGLK_HK}, giving
\begin{align*}
  \int_0^{2\pi} h(\theta) d\theta &=  \int_0^{2\pi}  \sum_{k=-K}^{K}\left( \frac{C_K}{2^{2K}}  \binom{2K}{K+k} \right) e^{ \imath k \theta} d\theta\\
  &\overset{(a)}{=} \int_0^{2\pi}\frac{C_K}{2^{2K}} \binom{2K}{K}  d\theta \\
	&=  \frac{C_K}{2^{2K-1}} \binom{2K}{K}  \pi,
\end{align*}
where (a) comes from the symmetry of $e^{ \imath k \theta}$ for $k \ne 0$.
So to make the $\cos^{2K}$ kernel integrate to one, we set
\begin{equation}
C_K = \frac{2^{2K-1}}{ \binom{2K}{K}  \pi}. \label{eq:CK}
\end{equation}

Thus, the formula for Fourier series coefficients of the FS-KDE using the $\cos^{2K}$ kernel is
\begin{equation}
F_k =   \frac{(K!)^2}{2 \pi N  (K-k)!(K+k)!} \sum_{n=0}^{N-1} w_n  e^{- \imath k \theta_n}.  \numberthis \label{eq:FS-KDE_FK_weighted}
\end{equation}

\subsubsection{Normal Approximation}
For large $K$s, the binomial coefficients in \eqref{eq:FS-KDE_FK} and \eqref{eq:CK} can be replaced with a normal approximation,
\begin{equation*}
\binom{n}{k} \approx \frac{2^n}{\sqrt{n \pi / 2}} e^{-\frac{(k-(n/2))^2}{n/2}},
\end{equation*}
giving an approximate version of \eqref{eq:FS-KDE_FK_weighted},
\begin{equation}
F_k =   \frac{1}{2 \pi N } e^{-k^2/K} \sum_{n=0}^{N-1} w_n e^{- \imath k \theta_n}  \numberthis \label{eq:FS-KDE_FK_approx}. \\
\end{equation}
This approximation saves computation as compared to \eqref{eq:FS-KDE_FK_weighted} and also reveals that the $F_ks$ decay exponentially.
The quality of the normal approximation improves as $K$ increases;
in our implementation we switch from \eqref{eq:FS-KDE_FK_weighted} to \eqref{eq:FS-KDE_FK_approx} when $2K \ge 80$.

\subsubsection{Truncation}
In our current formulation, the bandwidth of the kernel density estimate is controlled by $K$, which also governs how many Fourier series terms are nonzero.
Careful inspection of Figure~\ref{fig:kernels} reveals that the $\cos^{2K}$ kernel  does not sharpen quickly as $K$ increases: a sharp kernel requires a large $K$ and therefore a long descriptor.
One way to achieve sharp kernels with a shorter descriptor is through truncation.
The approximation \eqref{eq:FS-KDE_FK_approx} reveals that when $K$ is large, the decaying exponential term will cause $|F_k|$ to be very small for $k$ near $K$.
In fact, 
\begin{equation*}
\max_{\Theta}\frac{|F_k|}{|F_0|} =  e^{-k^2/K}.
\end{equation*}
Thus, for a fixed $K$ and a small truncation threshold $\epsilon$, we create a truncated FS-KDE, $\hat{F}$, according to
\begin{equation}
\label{eq:trunc}
\hat{F}_k = \begin{cases}
0 & \text{if } e^{-k^2/K} < \epsilon,\\
F_k & \text{otherwise}.
\end{cases}
\end{equation}
In the angular domain, truncation introduces distortion into the kernel, but this distortion is slight even when many coefficients are truncated (see Figure~\ref{fig:truncKernel}).
We provide MATLAB code for the FS-KDE using the $\cos^{2K}$ kernel in the reproducible research compendium for this article, \cite{McCannFK:15:web}.

\begin{figure}
\centering
\includegraphics[page=8]{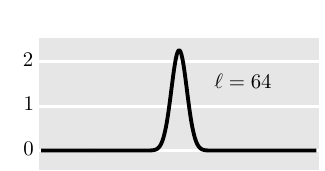}%
\includegraphics[page=9]{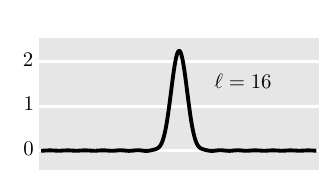}\\%
\includegraphics[page=10]{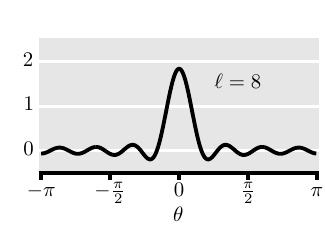}%
\includegraphics[page=11]{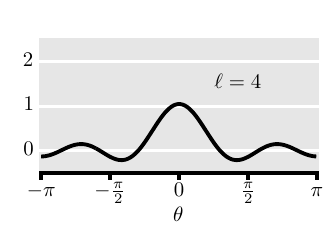}%
\caption{Examples of the $\cos^{2K}$ kernel  of order 64 with different levels of truncation, where $\ell$ is the number of non-zero coefficients.
Distortion is barely noticeable even when three quarters of the $F_k$s are set to zero (upper-right panel).}
\label{fig:truncKernel}
\end{figure}

\section{Canonicalization}
\label{sec:canonicalization}
We showed in Section~\ref{sec:props} that FS-KDEs are rotation invariant in the sense that a rotation of the input angles causes a corresponding rotation in the density estimate.
We may, however, also desire that a rotation of the input angles cause no change at all to the estimated distribution.
The would be useful if, e.g., FS-KDEs are being used as point descriptors in an image matching application. 
We can achieve this by rotating FS-KDEs to a standard, or \emph{canonical}, position, such that all FS-KDEs that are rotations of each other end up with the same canonical version.
In this section, we present two methods of achieving this canonicalization.

\subsection{$F_1$ Canonicalization}

A natural way of canonicalizing an angle-weight pair, $(\Theta, W)$, is to rotate the angles such that their mean is equal to zero.
One way to define the angular mean is to assign to each angle $\theta_n$ a complex number, $z_n = w_n e^{\imath \theta_n}$ with modulus $w_n$ and argument $\theta_n$, 
and then sum these numbers and take the argument of the result,
$\bar{\theta} = \arg \left( \sum_{n=0} ^ {N-1} z_n \right)$.
Then, the canonical angle-weight pair is $(\tilde{\Theta}, W)$, where $\tilde{\Theta} = \Theta_{-\bar{\theta}}$, which is the rotation of $\Theta$ by $-\bar{\theta}$.

From \eqref{eq:FS-KDE_FK} we see that the argument of the first Fourier series coefficient of the corresponding FS-KDE, $\arg (F_1)$, is equal to $-\bar{\theta}$ so long as $H_1$ is real and positive, as is the case for the $\cos^{2K}$ kernel.
As a result, canonicalizing the angle-weight pair causes $F_1$ to be real because $\bar{\theta} = 0$. 
Using this fact, we can directly canonicalize an FS-KDE, $f$, without having to know the angle-weight pair that it came from.
We define the canonical version of $f$ as
\begin{equation}
\tilde{f}(\theta) \overset{FS}{\longleftrightarrow} \tilde{F}_k = e^{-j k \arg (F_1) } F_k,
\label{eq:canon}
\end{equation}
which is the rotation of $f$ that makes $F_1$ real.

We now show that this canonicalization has the property that rotating a set of angles does not change its canonical FS-KDE.
In other words, all rotations of an angle-weight pair have the same canonical FS-KDE, $\tilde{f}$.

\begin{lemma}
Let $(\Theta, W)$ be an angle-weight pair, let $(\Theta_\phi, W)$ be its rotation by $\phi$, and let $F$ and $F_\phi$ be their FS-KDEs.
Then,
\begin{equation*}
\tilde{F}_{\phi,k} = \tilde{F}_k%
\end{equation*}%
\end{lemma}%
\begin{proof}
Beginning with an FS-KDE \eqref{eq:FS-KDE_FK}, we have
\begin{align*}
\arg (F_{\phi,1}) &= \arg \left( \frac{H_1 }{N} \sum_{n = 0}^{N-1} w_n   e^{- \imath  (\theta_n + \phi)} \right) \\
&= \arg \left( e^{-\imath \phi} \frac{H_1 }{N} \sum_{n = 0}^{N-1} w_n   e^{- \imath \theta_n} \right) \\
&= \arg (F_1) - \phi.
\end{align*}
We know from Section~\ref{sec:props} that $F_{\phi,k} = e^{-\imath k \phi} F_k$, and thus using the definition of canonicalization \eqref{eq:canon},
\begin{align*}
\tilde{F}_{\phi,k} &= e^{- \imath k \arg (F_{\phi, 1})} F_{\phi, k} \\
&=  e^{- \imath (\arg (F_1) - \phi) k} e^{-\imath k \phi} F_k \\
&= e^{-  \imath k \arg (F_1)} F_k = \tilde{F}_k.
\end{align*}
\end{proof}

\subsection{Stability of $F_1$ Canonicalization}
Now that we have shown that $F_1$ canonicalization aligns distributions that are exact rotations of each other, we study its effect on distributions that are noisy rotations of each other.
Intuitively, a good canonicalization will give similar canonical versions to all FS-KDEs that are noisy rotations of each other;
We call this property \emph{stability}.
Conversely, a bad canonicalization might amplify small amounts of noise, assigning similar FS-KDEs very different canonical versions.
The following theorem states that the stability of $F_1$ canonicalization is related to the magnitude of the first Fourier series coefficient of the distribution that is being canonicalized, $|F_1|$.
We leave the proof of the theorem to Appendix~\ref{app:proof}.

\begin{theorem}[Stability of $F_1$ Canonicalization]
\label{thm:canon}
Let $(\Theta, W)$ be an angle-weight pair.
Without loss of generality, assume $\bar{\theta}=0$ and $\sum_{n=0}^{N-1} w_n = N$. 
Let $(\noisy{\Theta}, \alpha \noisy{W})$ be its noisy version such that $ \noisy{w}_n e^{\imath \noisy{\theta}_n}  = w_n e^{\imath \theta_n} + \epsilon_n$ where the $\epsilon_n$ are drawn according to the complex normal distribution with mean zero and standard deviation $\sigma / \sqrt{N}$ (i.e., the imaginary part of $\epsilon_n\sim \mathcal{N}(0, \sigma^2 / N)$ and the real part of $\epsilon_n \sim \mathcal{N}(0, \sigma^2 / N)$), and $\alpha$ scales $\noisy{W}$ such that $\sum_{n=0}^{N-1} \alpha\noisy{w}_n = N$.
Let $f$ and $\noisy{f}$ be the $K$th-order $\cos^{2K}$ FS-KDEs of $(\Theta, W)$ and $(\noisy{\Theta}, \noisy{W})$, respectively.
Then
\begin{multline*}
\ev\left[ ||\noisy{F} - \noisy{\tilde{F}}||\right] \le  \\
\ev\left[ \sqrt{ \sum_{k=-K}^K \left(  2 B_k N  \sin  \left( \frac{k}{2} \arctan \left( \frac{  B_1 \epsilon }{ |F_1| +  B_1 \upsilon  } \right) \right)  \right)^2 } \right],
\end{multline*}
with $\epsilon$ $\sim \mathcal{N}(0, \sigma^2)$, $\upsilon$ $\sim \mathcal{N}(0, \sigma^2)$, and $B_k$ the coefficient from \eqref{eq:FS-KDE_FK}, $B_k = \frac{C_K}{2^{2K} N} \binom{2K}{K+k}$.
\end{theorem}

To make use of Theorem~\ref{thm:canon}, we note that the distance between a canonical distribution and its canonical noisy version is bounded by the distance due to noise and the distance due to canonicalizing the noisy version, i.e.  $||F - \noisy{\tilde{F}}|| \le ||F - \noisy{F}|| + ||\noisy{F} - \noisy{\tilde{F}}||$.
The theorem lets us calculate the expected value of $||\noisy{F} - \noisy{\tilde{F}}||$ only as a function of $|F_1|$ relative to the variance of the noise, $\sigma$, without having to know $f$ or $\noisy{f}$.
Notably, $\ev\left[||\noisy{F} - \noisy{\tilde{F}} || \right]$ approaches zero as $|F_1|$ grows relative to the noise. 
Because the norm is always non-negative, its expected value approaching zero implies that its variance is also approaching zero.
This means that as noise gets smaller, $||F - \noisy{\tilde{F}}|| \approx ||F - \noisy{F}|| \approx 0$, which is what we set out to show.

We illustrate this with a simulation (Figures~\ref{fig:canonBoundGood} and \ref{fig:canonBoundBad}).
We first generate two random distributions, one with a large $|F_1|$ and one with a small $|F_1|$.
For each of these distributions, we generate noisy versions for a range of noise levels and calculate $||F - \noisy{F} ||$ and $||\noisy{F} - \tilde{\noisy{F}} ||$.
Comparing Figures~\ref{fig:canonBoundGood} and \ref{fig:canonBoundBad}, we see that $||\noisy{F} - \tilde{\noisy{F}} ||$ is expected to be smaller in Figure~\ref{fig:canonBoundGood}, where $|F_1|$ is large.
For comparison, we plot the distance caused by rotating these distributions, $||F - F_\phi||$ (Figure~\ref{fig:canonBoundGood}(e) and \ref{fig:canonBoundBad}(e)).
In Figure~\ref{fig:canonBoundGood}, the distance caused by rotation is larger than the distance due to canonicalizing the noisy versions, but in Figure~\ref{fig:canonBoundBad}, the distance due to canonicalization is significant compared to the rotation distance.

\begin{figure}
  \centering 

    \subfloat[]{\includegraphics[page=12]{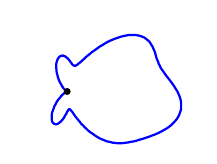}}\hspace{1cm}
    \subfloat[]{\includegraphics[page=13]{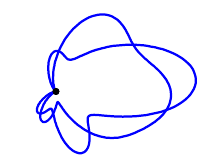}}\\			
		\subfloat[]{\includegraphics[page=14]{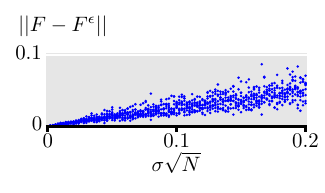}}%
		\subfloat[]{\includegraphics[page=15]{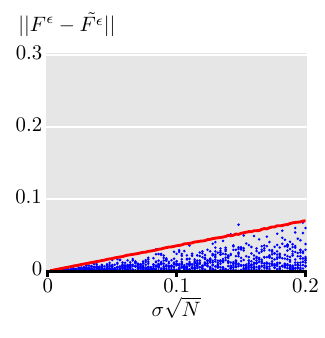}}\\
  \subfloat[]{\includegraphics[page=16]{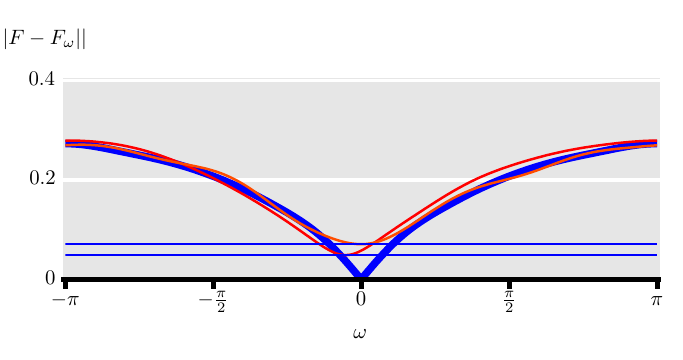}}
  \caption{(a) A distribution, $f$.
    (b) Two of its noisy versions, $\noisy{f}$.
    (c) The distance between $f$ and its noisy versions plotted as a function of increasing noise.
    (d) The additional error introduced by $F_1$ canonicalization of the noisy versions, along with the expectation from Theorem~\ref{thm:canon}.
    Because this distribution has a large $|F_1|$ value compared to the noise, $F_1$ canonicalization makes small changes to noisy versions of $f$.
    (e) Curves indicate $||F - F_{\phi}||$ (bold line) and $||F - \noisy{F}_{\phi}||$ (thin lines) for  $f$, two noisy versions of $f$, and their rotations by $\phi$.
    Because $|F_1|$ is large compared to the noise, $||F - \tilde{\noisy{F_\phi}}||$ (horizontal lines) is almost as small as $\min_\phi  ||F - \noisy{F}_{\phi}||$.}
  \label{fig:canonBoundGood}
\end{figure}

\begin{figure}
    \centering
    \subfloat[]{\includegraphics[page=17]{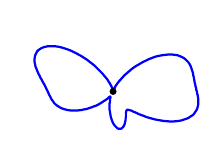}}\hspace{1cm}
    \subfloat[]{\includegraphics[page=18]{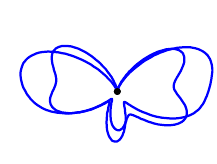}}\\			
		\subfloat[]{\includegraphics[page=19]{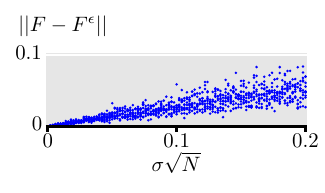}}%
		\subfloat[]{\includegraphics[page=20]{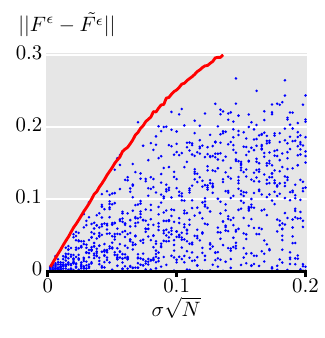}}\\
  \subfloat[]{\includegraphics[page=21]{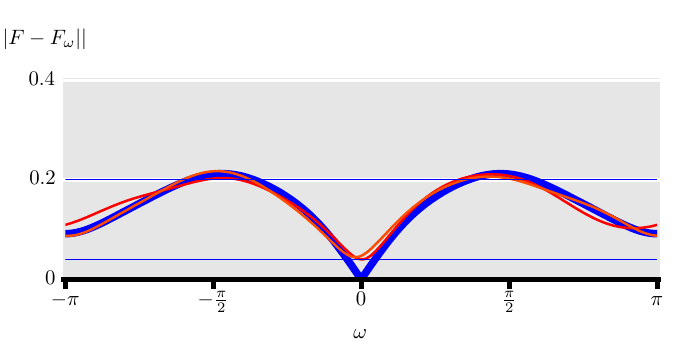}}
  \caption{(a) A distribution, $f$.
    (b) Two of its noisy versions, $\noisy{f}$.
    (c) The distance between $f$ and its noisy versions plotted as a function of increasing noise.
    (d) The additional error introduced by $F_1$ canonicalization of the noisy versions, along with the expectation from Theorem~\ref{thm:canon}.
    Because this distribution has a small $|F_1|$ value compared to the noise, $F_1$ canonicalization may make large changes to noisy versions of $f$.
    (e) Curves indicate $||F - F_{\phi}||$ (bold line) and $||F - \noisy{F}_{\phi}||$ (thin lines) for $f$, two noisy versions of $f$, and their rotations by $\phi$.
    Because $|F_1|$ is small compared to the noise, $||F - \tilde{\noisy{F_\phi}}||$ (horizontal lines) is sometimes large.}
  \label{fig:canonBoundBad}
\end{figure}

As a concrete example, take a patch matching application, such as we describe in Section~\ref{sec:ex_patch}.
If distributions in a dataset are randomly rotated and $|F_1|$ for each patch is large relative to the expected noise, it makes sense to canonicalize the distributions before matching because much of the distance between corresponding patches will come from their rotation, which canonicalization will remove.
If, on the other hand, patches in the dataset are not rotated, canonicalization will hurt performance because $||\noisy{F} - \tilde{\noisy{F}}||$ will increase the distance between matching patches.
As $|F_1|$ for the patches shrinks relative to the noise, canonicalization becomes increasingly unstable.
This is because when $|F_1|$ is small, a small amount of noise can greatly affect $\arg (F_1)$.
This situation can arise in two ways.
The first is when all weights are small, meaning the distribution being calculated is essentially zero; instability is no problem in this case because rotation has no effect on FS-KDEs that are nearly zero.
The second is when the distribution has symmetry, e.g., when $\Theta$ contains only angles only at zero and $\pi$.
Such cases may arise in practice, leading us to explore a generalization of $F_1$ canonicalization that can remove these symmetries. 

\subsection{$F_k$ Canonicalization}
We can generalize the idea in \eqref{eq:canon} to rotating $f$ by an angle, $(\arg(F_{\ell}) + 2\pi n) / \ell$, such that $F_\ell$ is real.
The added complexity is that for $\ell \ge 2$, this angle is not unique; it can take $\ell$ different values.
We disambiguate these by defining $F_k$ canonicalization recursively,
\begin{equation*}
\tilde{f}^\ell(\theta) \overset{FS}{\longleftrightarrow} \tilde{F}_k^\ell = e^{-j k \arg(F_{\ell}) / \ell} \tilde{F}_k^{\ell-1},
\end{equation*}
with $\tilde{F}_k^1  = \tilde{F}_k$ as defined in \eqref{eq:canon}.
One way to think about this process is that we first $F_1$ canonicalize, then we pick the smallest rotation that makes $F_2$ real, then pick the smallest rotation that makes $F_3$ real, and so on until $F_\ell$.
For any choice, $1 \le \ell \le K$, we can show that rotating the input set of angles does not affect the canonical version, using the same steps as for $F_1$ canonicalization.

The benefit of using $\ell > 1$ is that for angular distributions with a certain kind of symmetry, $|F_1|$ may be small (and thus $F_1$ canonicalization will not be robust to noise), while, e.g., $|F_2|$ may be large, meaning $F_2$ canonicalization will be robust to noise.
The trade-off is that if $|F_1|$ and $|F_2|$ are of similar size, $F_1$ canonicalization will be more robust to noise.
(To see this, note that $F_2$ canonicalization is just another mean subtraction, except that the mean is calculated by first doubling all the angles in $\Theta$.
This doubling can remove unwanted symmetry, but it also amplifies noise.)

In our experiments, we leverage this in the following way:
When what is important is pairwise distances between FS-KDEs, then we can define a canonical distance,
\begin{equation*}
||f - g||_{\text{canonical}} = 2\pi \min_{\phi} ||F - e^{-j k \phi }G||.
\end{equation*}
Finding this distance requires an optimization over $\phi$, so is not appropriate when many pairwise distances must be computed.
A reasonable approximation, however, is
\begin{equation}
  ||f - g||_{F_k \text{ canonical}} = \min_{1 \le \ell \le K} ||\tilde{f}^\ell - \tilde{g}^\ell||,
  \label{eq:canon_dist_all}
\end{equation}
which only requires the calculation of $K$ distances.


\section{Experiments and Discussion}
\label{sec:experiments}
We now present experiments in keypoint description, person detection, and texture segmentation that show the promise of FS-KDEs using the $\cos^{2K}$ kernel as a tool in image processing.

\subsection{Keypoint Description}
\label{sec:ex_patch}
A typical approach to image registration involves selecting keypoints from the images to be matched, finding pairs of corresponding keypoints, and solving for the transform based on the location of these pairs.
One way to find correspondences between keypoints is to calculate a keypoint descriptor from the pixels around each keypoint.
When two keypoints correspond, the distance between their descriptors should be low; when they do not, it should be high.
A good keypoint descriptor should be highly discriminative while simultaneously being invariant to the transform that the registration aims to reverse.

We evaluate the performance of the $\cos^{2K}$ kernel  as a keypoint descriptor using the University of British Columbia Multi-view Stereo Correspondence Dataset~\cite{WinderHB:09}.
This dataset was constructed by extracting image patches around difference of Gaussian interest points in many images of the same few scenes (Statue of Liberty, Yosemite National Park, and Notre-Dame Cathedral).
Depth maps of the scenes were used to determine which interest points match in 3D space, resulting in lists of corresponding and non-corresponding image patches (see Figure~\ref{fig:exPatches}).
The patches are $64 \times 64$ greyscale images; in our experiments we crop them to a circular region with a diameter of 60 pixels to avoid artifacts when rotating the patches.

\begin{figure}
\centering
\subfloat[Corresponding]{\includegraphics[width=.45\columnwidth]{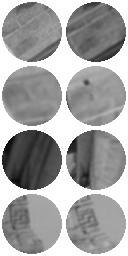}} \quad
\subfloat[Non-corresponding]{\includegraphics[width=.45\columnwidth]{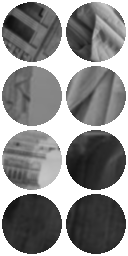}}
\caption{Examples of corresponding (left) and non-corresponding (right) pairs of patches from the British Columbia Multi-view Stereo Correspondence Dataset~\cite{WinderHB:09}.
Though the patches are rotated to a canonical orientation, corresponding patches still exhibit viewpoint and intensity variation.
}
\label{fig:exPatches}
\end{figure}

We compare three simple keypoint descriptors: \begin{inparaenum}[(i)]
\item The \emph{raw intensity} descriptor is formed by concatenating the pixel values of the patch into a vector.
It has dimension equal to the number of pixels in the patch, 2,828.
\item The \emph{gradient histogram} descriptor is formed by computing the image gradient at each pixel in the patch and forming a histogram of the gradient angles weighted by the norm of the gradient.
The dimension of the descriptor is equal to the number of histogram bins; we vary it between 4 and 32.
\item The \emph{$\cos^{2K}$ } descriptor is formed by computing the image gradient at each pixel and computing a $\cos^{2K}$ FS-KDE of the gradient angles weighted by the norm of the gradient. 
We truncate the descriptors according to \eqref{eq:trunc} with $\epsilon$ = $1 \times 10^{-5}$ and we vary the descriptor length between 4 and 32. \end{inparaenum}

Additionally, we compare three canonical versions of these descriptors: \begin{inparaenum}[(i)]
\item The \emph{canonical gradient histogram}, which is the same as the gradient histogram but with its bins rotated so that the first bin has the largest value.
\item The \emph{$F_1$ canonical $\cos^{2K}$}, which follows the canonicalization procedure from \eqref{eq:canon}.
\item The \emph{$F_{k}$ canonical $\cos^{2K}$}, which follows the canonicalization procedure from \eqref{eq:canon_dist_all}.
\end{inparaenum}

We use each of these methods to compute descriptors for 250,000 pairs of corresponding patches and 250,000 pairs of non-corresponding patches.
We then calculate the Euclidean distance between each pair of descriptors;
good descriptors should assign small distances to corresponding patches and large distances to non-corresponding patches.
Setting a threshold on this distance allows the descriptor to classify pairs of patches as corresponding or not.
To quantify the performance of each descriptor, we calculate the area under its ROC curve (the curve formed when plotting true positive rate versus false positive rate over the whole range of possible threshold values).
A perfect classifier has an area under ROC (AUC) of 1, while a random classifier has an an AUC of .5.


Figure~\ref{fig:patchError} shows the results of our comparison of keypoint descriptors in terms of AUC.
The $\cos^{2K}$ descriptor has the highest AUC for a wide range of descriptor sizes (6 to 26) and has the highest overall AUC of .83 at descriptor length 10).
After this peak at descriptor length ten, its performance declines as the descriptor length increases, which is consistent with the idea that the descriptor becomes overly specific when it is long, increasing distances between corresponding patches.
The error chance of the gradient histogram descriptor also increases with descriptor size, for the same reason.
We attribute the gap in performance between the $\cos^{2K}$ and gradient histogram descriptors to the $\cos^{2K}$ kernel's smooth handling of small rotations:
even though the patches in the dataset are rotated to a canonical orientation, small rotations do exist between corresponding patches, which could distort the gradient histograms (see Figure~\ref{fig:exPatches} for examples).
The canonical versions of the $\cos^{2K}$ and histogram descriptors perform generally worse than their non-canonical counterparts, which is as expected since no canonicalization should be necessary for this dataset.
In this case, canonicalization will decrease the distance between non-corresponding patches more than it does for corresponding patches, increasing the number of false positives at a given threshold.


\begin{figure}[htb]
\centering
\includegraphics[page=22]{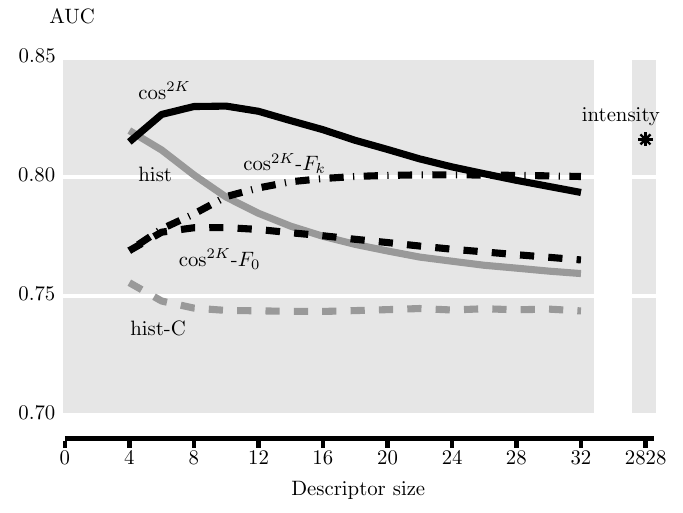}
\caption{AUC for intensity, gradient histogram, and the proposed $\cos^{2K}$ descriptors of varying size.
The best performance is achieved by the $\cos^{2K}$ descriptor of size ten.
The canonical descriptors perform poorly in this experiment because the patches are already in a canonical orientation.}
\label{fig:patchError}
\end{figure}

In a separate experiment, we randomly rotated each patch in the dataset and ran the same comparison (Figure~\ref{fig:patchErrorRot}).
As expected, this greatly increased the error rate for the histogram and $\cos^{2K}$ descriptors, as they are not rotation invariant without canonicalization.
The canonical versions of these descriptors were mostly unaffected by the change, since they are rotation invariant descriptors.
Both canonical versions of the $\cos^{2K}$ descriptor were superior to the canonical gradient histogram, which we hypothesize is due to the robustness to noise of the proposed canonicalizations.
The $F_{k}$ canonicalization was better than the $F_1$ canonicalization, which suggests that symmetries of the type discussed in Section~\ref{sec:canonicalization}, which break the $F_1$ canonicalization, do exist in this dataset.

\begin{figure}[htb]
\centering
 \includegraphics[page=23]{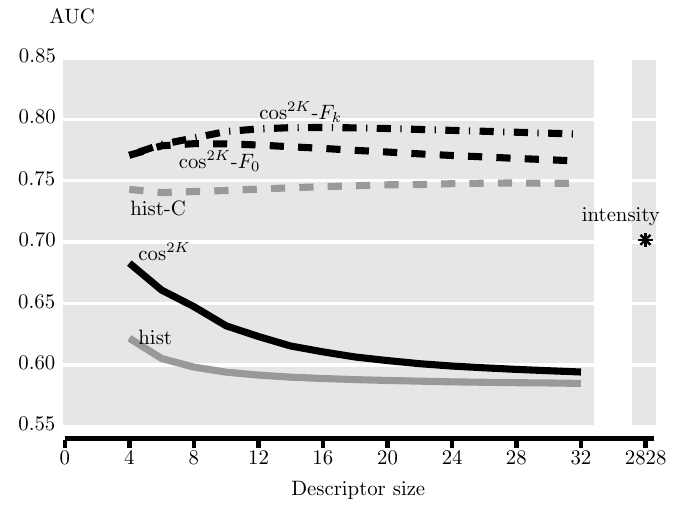}
\caption{AUC for intensity, gradient histogram, and the proposed $\cos^{2K}$ descriptors of varying size on patches with random orientation.
The $\cos^{2K}$ FS-KDE and gradient histogram descriptors cannot handle patch rotations, so have much higher error chance here than in Figure~\ref{fig:patchError}.
The best performance is now achieved by the $\cos^{2K}$ FS-KDE descriptor with $F_{k}$ canonicalization.}
\label{fig:patchErrorRot}
\end{figure}


\subsection{Person Detection}
We now evaluate the usefulness of the FS-KDE using the $\cos^{2K}$ kernel as a feature in a person detection application.
A typical approach to person detection (or, in general, object detection), is to train a classifier on features which consist of distributions of angles.
To preserve some spatial information, these distributions are calculated in a few windows of the input, e.g., upper left, upper right, lower left, lower right, and then concatenated together to form the final feature vector.

We use the INRIA person dataset~\cite{DalalT:05} to do a comparison of feature detectors for human detection.
This dataset is intended for supervised classification of images as containing a person (positive) and not containing a person (negative).
It includes a training set of 2,416 positive and 1,218 negative images and a testing set of 1,126 positive and 453 negative images.
For all images, we use the center $64 \times 128$ pixels for feature extraction.

We compare the following feature extractors:
\begin{inparaenum}[(i)]
\item \emph{Raw intensity} simply uses the pixels of the image as features and therefore has length 8192.
\item \emph{Gradient histogram} separates the image into $8\times8$-pixel blocks and computes a histogram of gradients inside each block, concatenating these histograms into a feature vector. We vary the number of bins per block from 4 to 64.
\item \emph{HOG}, originally described in \cite{DalalT:05}, also forms gradient histograms from blocks of the input image, but includes an additional block normalization step that can increase the feature's illumination invariance.
We use the implementation in \cite{VedaldiF:08} and vary the number of orientations per block from four to sixty-four.
\item \emph{$\cos^{2K}$}  forms truncated ($\epsilon = 1 \times 10^{-5}$) $\cos^{2K}$ FS-KDEs for $8 \times 8$-pixel blocks of the input. We vary the descriptor length per block from four to sixty-four.
\end{inparaenum}

We use each of these methods to extract features from the training and test sets.
For each set of features, we train a linear SVM classifier (in MATLAB) on the training set, then use it to classify each image in the test set as negative or positive.
In a separate experiment, we create new testing and training datasets by rotating each image between $15$ and $-15$ degrees uniformly at random.%
\footnote{This rotation does not produce edge artifacts because we crop the central portion of a larger image to form the training and testing images.}
This rotation is not enough that canonicalization is necessary, but is meant to test the robustness of the features to small rotations.
We report accuracy, the number of correctly classified images in the testing set divided by the total number of images in the testing set, from both experiments in Figure \ref{fig:personAcc}.


\begin{figure}
\centering
\includegraphics[page=24]{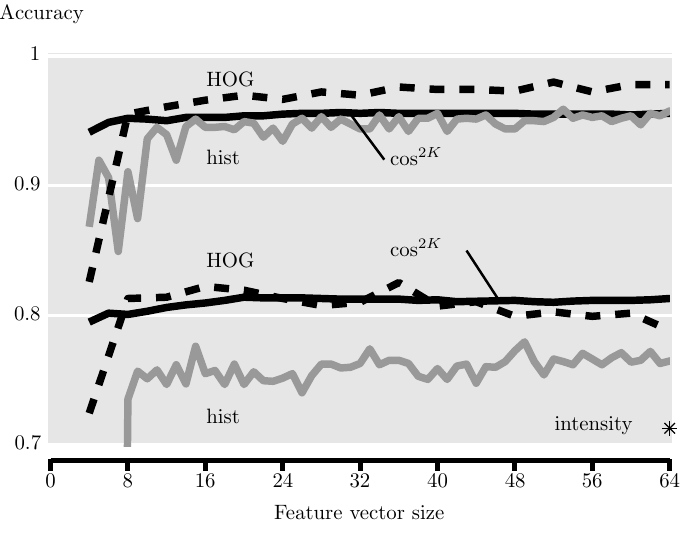}
\caption{Accuracy on a person detection task for the intensity, gradient histogram, HOG, and $\cos^{2K}$ feature extractors, plotted as a function of the feature vector length per image block.
The top set of lines is for the INRIA person dataset~\cite{DalalT:05}, the bottom set is for the same dataset with small random rotations added.
Without rotations, HOG features have the highest accuracy for most lengths and the intensity features (as expected) have the lowest.
With rotations, the performance of all four methods declines, but the $\cos^{2K}$ FS-KDE declines the least, leaving it with accuracy comparable to HOG.
The accuracy of the intensity features for the rotated dataset was below chance and not plotted.
}
\label{fig:personAcc}
\end{figure}


For the unrotated set, the HOG features have the highest accuracy, except when the feature vector size is very small (four numbers per image block, resulting in 512 numbers per image).
We suspect HOG's increase in performance over the gradient histogram and KDE features comes from the normalization scheme used in HOG, which gives it an invariance to illumination changes missing in the other methods.
The low accuracy of the intensity features is as expected, given that greyscale intensity is not a reliable way to distinguish people from background clutter.
The accuracy of the histogram gradient and both KDEs features are similar, except that the accuracy of the histogram features is less stable as the feature vector size changes.
We attribute this to the binning effects introduced by the histogram.
We also note that the decline in performance as descriptor length increases seen in Figure~\ref{fig:patchError} is not evident here because we use an SVM as opposed to simply calculating distances.

When a small amount ($\pm$ 15 degrees) of rotation is added to the images in the dataset, the accuracy of all the feature sets decreases, but the decrease is smallest for the $\cos^{2K}$ features.
The rotation makes intensity features worse than chance (not plotted in Figure~\ref{fig:personAcc}) because these features have no invariance to rotation.
We suspect that binning artifacts (as discussed in Figure~\ref{fig:binning}) explain the relatively larger decrease in accuracy for the gradient histogram and HOG features, because they both rely on gradient histograms.
This experiment shows that the smooth $\cos^{2K}$ kernel provides greater invariance to small rotations than the binning employed by histograms, resulting in higher accuracy in the person detection task.


\subsection{Texture Segmentation}
Distributions of angles are also useful as texture features.
In our previous work, \cite{McCannMFCOK:14}, we presented an algorithm for segmentation based on unmixing the local color histograms of an input image, which we call the Occlusion of Random Textures SEGmenter, (ORTSEG).
In this experiment, we extend ORTSEG to include distributions of angles as well.
We compare the effectiveness of histograms versus $\cos^{2K}$ FS-KDEs to capture these distributions of angles.

We compare the methods on the \emph{random texture dataset} from \cite{McCannMFCOK:14} plus an additional synthetic dataset, which we refer to as \emph{dead leaves}.
The images in the random texture dataset set each comprise three textures with different color distributions and no meaningful edge information (see \cite{McCannMFCOK:14} for more details and examples).
The dead leaves dataset (Figure~\ref{fig:deadLeavesEx}) images each comprise three textures with the same color distributions but differently oriented edges.
To create this dataset, we first pick three seed locations at random and use them to partition the image into the three Voronoi regions.
We then generate the image via a dead leaves procedure: we sequentially place shapes of random color into the image at random locations until every pixel is covered.
Depending on which of the three regions a shape lands in, it is selected to be either a vertical bar, horizontal bar, or diagonal bar.
We select the ground truth label of each pixel to correspond to the shape that covered it most recently.

\begin{figure}[htb]
  \centering
  \subfloat[Input]{\includegraphics[width=.45\columnwidth]{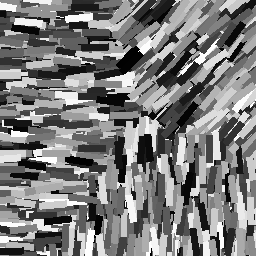}}
  \quad
  \subfloat[Ground truth]{\includegraphics[width=.45\columnwidth]{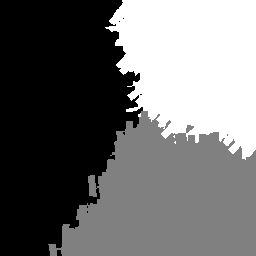}}
  \caption{An example image and corresponding ground truth from the dead leaves dataset.
  For images like these, angular distributions are an important feature.}
  \label{fig:deadLeavesEx}
\end{figure}

We compare three methods on this dataset.
\begin{inparaenum}[(i)]
\item \emph{ORTSEG} is the original segmentation system described in \cite{McCannMFCOK:14}, which relies only on color histograms.
\item \emph{ORTSEG-hist} uses both color histograms and local gradient histograms, with the number of bins selected from training between eight and 40.
\item \emph{ORTSEG-FS-KDE} uses color histograms and local $\cos^{2K}$ FS-KDEs of gradients, with no canonicalization, and with the number of complex coefficients selected from training between four and 20.
\end{inparaenum}
We do not evaluate canonical versions of these methods because canonicalization will make the angular distributions in the different texture regions of the dead leaves images match, resulting in low segmentation accuracy.
The choice of whether or not to canonicalize for segmentation depends on whether textures that match except for their orientation are intended to be grouped together or not.

The experiment is structured exactly as in \cite{McCannMFCOK:14};
in short, it is a leave-one-out cross validation.
The results are reported in terms of Rand index~\cite{Rand:71,UnnikrishnanPH:07}, which measures the fraction of pairs of pixels that are either in the same region in both the segmentation result and ground truth or in different regions in both the segmentation result and ground truth.
It therefore ranges between zero and one, with one being perfect agreement with the ground truth.

The results of the segmentation experiment are given in Table~\ref{tab:segResults}.
The three methods perform equally well on the random texture dataset, which makes sense because color information alone is enough to distinguish the textures.
On the dead leaves dataset, the basic ORTSEG method, which relies only on color, cannot distinguish the textures at all and thus performs poorly.
The gradient histogram and FS-KDE versions of ORTSEG improve performance by including edge information.
That performance increase is most pronounced for the FS-KDE version.
We attribute this difference to the smooth $\cos^{2K}$ kernel used in the FS-KDE giving better robustness to small variations in gradient angle as compared to histograms.
These results serve as a proof of concept for the efficacy of the $\cos^{2K}$ FS-KDE for including gradient information into the segmentation method ORTSEG.

\begin{table}
\renewcommand{\arraystretch}{1.1}
\centering
\caption{Comparison of the basic ORTSEG method with versions using gradient histograms and $\cos^{2K}$ FS-KDEs.
The augmented versions improve the performance on the dead leaves dataset, where edge information is critical.}
\begin{tabular}{l c c}
\toprule
& \multicolumn{2}{c}{\bf Dataset} \\
\cmidrule{2-3}
\bf Method & random texture & dead leaves\\
basic & {\bf 0.989 }$\pm$ 0.002 & 0.551 $\pm$ 0.087\\
hist & 0.988 $\pm$ 0.002 & 0.702 $\pm$ 0.161\\
FS-KDE & 0.989 $\pm$ 0.002 & {\bf 0.907 }$\pm$ 0.094\\
\bottomrule
\end{tabular}
\label{tab:segResults}
\end{table}

\section{Conclusion}
\label{sec:conc}
In this work, we presented a new bandlimited Gaussian-like kernel, useful for describing angular distributions in computer vision applications.
Because the kernel is bandlimited, the resulting KDEs are also bandlimited and therefore can be represented exactly by a finite number of their Fourier series coefficients, a technique which we call FS-KDE.
Though this type of density estimation is not new, it has not been much explored in image processing, where finite-length angular descriptors are very useful.
We also presented a canonicalization scheme for FS-KDEs which allows them to create rotation invariant descriptions of angular distributions and analyzed the robustness of this scheme to noise.

In our experiments, we compared FS-KDEs using our proposed kernel to histograms in the contexts of patch matching, person detection, and texture segmentation.
In the patch matching experiment, the FS-KDE descriptors outperformed histogram-based descriptors both when patches were upright and when they were randomly oriented.
The person detection experiment showed that FS-KDE features provide higher person detection accuracy than histogram features, especially when a small amount of random rotation was added to the dataset.
Finally, the segmentation experiment suggested that the FS-KDE is a better way to capture texture information than histograms in the context of texture segmentation. 
Taken together, these experiments provide strong proof of concept for the efficacy of FS-KDEs using our new bandlimited kernel as tools for describing distributions of angles in image processing applications.

\section{Acknowledgements}
The authors gratefully acknowledge support from the NSF through awards 0946825 and 1017278,
the Achievement Rewards for College Scientists Foundation Scholarship,
the John and Claire Bertucci Graduate Fellowship,
the Philip and Marsha Dowd Teaching Fellowship, and the CMU Carnegie
Institute of Technology Infrastructure Award.
\appendix

\section{Proof of Theorem~\ref{thm:canon}}
\label{app:proof}

\begin{proof}
\begin{align*}
  || \noisy{F} - \noisy{\tilde{F}} || &\overset{(a)}{=} \sqrt{ \sum_{k=-K}^K \left| \noisy{F_k} - e^{-\imath k \arg( \noisy{F}_1 )} \noisy{F_k} \right|^2 } \\ 
  &\overset{(b)}{=} \sqrt{ \sum_{k=-K}^K \left( |\noisy{F_k}| \left| 1 - e^{-\imath k \arg( \noisy{F}_1 ) }  \right| \right)^2 } \\ 
  &\overset{(c)}{=} \sqrt{ \sum_{k=-K}^K \left( |\noisy{F_k}| 2 \sin \left( \frac{k}{2} \arg(\noisy{F_1}) \right)  \right)^2 }, \numberthis \label{eq:p1} 
\end{align*}
where (a) follows from the definition of norm, (b) from factoring, and (c) from Euler's formula and the fact that $|e^{\imath \theta} x| = |x|$ for all $x$.

In order to find $\arg(\noisy{F_1})$ in \eqref{eq:p1}, we note that $\arg(\noisy{F_1}) = \arg(\noisy{F_1}/\alpha)$ and that
\begin{equation*}
  \noisy{F_1}/\alpha = B_1  \sum_{n=0}^{N-1}  w_n e^{ -\imath  \theta_n } + \epsilon_n = F_1 + B_1  \sum_{n=0}^{N-1} \epsilon_n
\end{equation*}
by assumption and the definition of the FS-KDE~\eqref{eq:FS-KDE_FK}.
Therefore
\begin{align*}
  \arg( \noisy{F_1} ) &= \arctan \left( \frac{\mathfrak{I}( \noisy{F_1}/\alpha ) }{ \mathfrak{R}( \noisy{F_1}/\alpha ) } \right) \\
                      &= \arctan \left( \frac{ B_1  \sum_{n=0}^{N-1} \mathfrak{I} ( \epsilon_n ) }{ |F_1| + B_1  \sum_{n=0}^{N-1} \mathfrak{R} ( \epsilon_n ) } \right), \numberthis \label{eq:p2}
\end{align*}
where $ \mathfrak{R}(z)$ and $ \mathfrak{I}$ denote the real and imaginary parts of $z$, respectively, because $F_1$ is real by assumption.

To bound $|\noisy{F_k}|$, we use the fact that $|B_k \sum_{n=0}^{N-1} w_n e^{ -\imath k  \theta_n }| \le |B_k \sum_{n=0}^{N-1} w_n|$ for any choice of $(\Theta, W)$, meaning that
\begin{equation*}
  |\noisy{F_k}| \le B_k N.
\end{equation*}
To finish the proof, we replace the sums in \eqref{eq:p2} with new random variables $\epsilon$ and $\upsilon$, the distributions of which we know because of the noise model assumed in the proof.

\end{proof}

\bibliographystyle{IEEEtran}
\bibliography{paper}  

\end{document}